\documentclass[review,11pt]{ReportTemplate}

\usepackage{bm}
\usepackage{paralist,algorithmic,algorithm}
\usepackage{amsmath,mathrsfs,amssymb}
\usepackage{url, graphicx, subfigure}

\newenvironment{proof}{\textbf{Proof:}\ }{\hspace{\stretch{1}}$\square$\\}

\newtheorem{theorem}{Theorem}

\DeclareFixedFont{\myfont}{OT1}{ptm}{m}{n}{7pt}
\DeclareFixedFont{\myfontb}{OT1}{ptm}{bx}{n}{7pt}

\DeclareMathOperator*{\argmin}{argmin}

\def \0 {\boldsymbol{0}}
\def \e {\boldsymbol{e}}
\def \p {\boldsymbol{p}}
\def \u {\boldsymbol{u}}
\def \v {\boldsymbol{v}}
\def \w {\boldsymbol{w}}
\def \x {\boldsymbol{x}}
\def \y {\boldsymbol{y}}
\def \z {\boldsymbol{z}}

\def \A {\boldsymbol{A}}
\def \G {\boldsymbol{G}}
\def \H {\boldsymbol{H}}
\def \I {\boldsymbol{I}}
\def \Q {\boldsymbol{Q}}
\def \X {\boldsymbol{X}}
\def \Y {\boldsymbol{Y}}

\def \wbar {\bar{\boldsymbol{w}}}

\def \alphavec {\boldsymbol{\alpha}}
\def \betavec {\boldsymbol{\beta}}
\def \etavec {\boldsymbol{\eta}}

\def \xivec {\boldsymbol{\xi}}

\def \Dcal {\mathcal{D}}
\def \Lcal {\mathcal{L}}

\def \Xcal {\mathcal{X}}
\def \Ycal {\mathcal{Y}}

\def \Rbb {\mathbb{R}}

\begin{document}
\begin{frontmatter}
\title{Large Margin Distribution Machine}

\author{Teng Zhang}
\author{Zhi-Hua Zhou\corref{cor1}}
\address{National Key Laboratory for Novel Software Technology\\
Nanjing University, Nanjing 210023, China} \cortext[cor1]{\small Corresponding author.
Email: zhouzh@nju.edu.cn}

\begin{abstract}
Support vector machine (SVM) has been one of the most popular learning algorithms, with the central idea of maximizing the \textit{minimum margin}, i.e., the smallest distance from the instances to the classification boundary. Recent theoretical results, however, disclosed that maximizing the minimum margin does not necessarily lead to better generalization performances, and instead, the margin distribution has been proven to be more crucial. In this paper, we propose the Large margin Distribution Machine (LDM), which tries to achieve a better generalization performance by optimizing the margin distribution. We characterize the margin distribution by the first- and second-order statistics, i.e., the margin mean and variance. The LDM is a general learning approach which can be used in any place where SVM can be applied, and its superiority is verified both theoretically and empirically in this paper.
\end{abstract}


\end{frontmatter}

\section{Introduction}

Support Vector Machine (SVM) \cite{Cortes1995Support,Vapnik1995The} has always been one of the most successful learning algorithms. The basic idea is to identify a classification boundary having a large margin for all the training examples, and the resultant optimization can be accomplished by a quadratic programming (QP) problem. Although SVMs have a long history of literatures, there are still great efforts \cite{Lacoste2013Block,Learning2013Cotter,Takac2013Mini,Jose2013Local,Convex2013Do} on improving SVMs.

It is well known that SVM can be viewed as a learning approach trying to maximize over training examples the \textit{minimum margin}, i.e., the smallest distance from the examples to the classification boundary, and the margin theory \cite{Vapnik1995The} provided a good support to the generalization performance of SVM. It is noteworthy that the margin theory not only plays an important role for SVMs, but also has been extended to interpret the good generalization of many other learning approaches, such as AdaBoost \cite{Freund1995A}, a major representative of ensemble methods \cite{Zhou2012Ensemble}. Specifically, Schapire et al. \cite{Schapire1998Boosting} first suggested margin theory to explain the phenomenon that AdaBoost seems resistant to overfitting; soon after, Breiman \cite{Breiman1999Prediction} indicated that the minimum margin is crucial and developed a boosting-style algorithm, Arc-gv, which is able to maximize the minimum margin but with a poor generalization performance. Later, Reyzin et al. \cite{Reyzin2006How} found that although Arc-gv tends to produce larger minimum margin, it suffers from a poor margin distribution; they conjectured that the margin distribution, rather than the minimum margin, is more crucial to the generalization performance. Such a conjecture has been theoretically studied \cite{Wang2011A,Gao2012On}, and it was recently proven by Gao and Zhou \cite{Gao2012On}. Moreover, it was disclosed that rather than simply considering a single-point margin, both the margin mean and variance are important \cite{Gao2012On}. All these theoretical studies, however, focused on boosting-style algorithms, whereas the influence of the margin distribution for SVMs in practice has not been well exploited.

In this paper, we propose the Large margin Distribution Machine (LDM), which tries to achieve strong generalization performance by optimizing the margin distribution. Inspired by the recent theoretical result \cite{Gao2012On}, we characterize the margin distribution by the first- and second-order statistics, and try to maximize the margin mean and minimize the margin variance simultaneously. For optimization, we propose a dual coordinate descent method for kernel LDM, and propose an averaged stochastic gradient descent (ASGD) method for large scale linear kernel LDM. Comprehensive experiments on twenty regular scale data sets and twelve large scale data sets show the superiority of LDM to SVM and many state-of-the-art methods, verifying that the margin distribution is more crucial for SVM-style learning approaches than minimum margin.

The rest of this paper is organized as follows. Section 2 introduces some preliminaries. Section 3 presents the LDM. Section 4 reports on our experiments. Section 5 discusses about some related works. Finally, Section 6 concludes.

\section{Preliminaries} \label{sec:preliminaries}

We denote by $\Xcal \in \Rbb^d$ the instance space and $\Ycal =\{+1,-1\}$ the label set. Let $\Dcal$ be an unknown (underlying) distribution over $\Xcal \times \Ycal$. A training set of size $m$
\begin{align*}
S = \{(\x_1, y_1), (\x_2, y_2), \dots, (\x_m, y_m)\},
\end{align*}
is drawn identically and independently (i.i.d.) according to the distribution $\Dcal$. Our goal is to learn a function which is used to predict the labels for future unseen instances.

For SVMs, $f$ is regarded as a linear model, i.e., $f(\x) = \w^\top \phi(\x)$ where $\w$ is a linear predictor, $\phi(\x)$ is a feature mapping of $\x$ induced by a kernel $k$, i.e., $k(\x_i, \x_j) = \phi(\x_i)^\top \phi(\x_j)$. According to \cite{Cortes1995Support,Vapnik1995The}, the margin of instance $(\x_i, y_i)$ is formulated as
\begin{align} \label{defn:margin}
\gamma_i = y_i \w^\top \phi(\x_i), \forall i = 1,\ldots, m.
\end{align}
From \cite{Cristianini2000introduction}, it is shown that in separable cases where the training examples can be separated with the zero error, SVM with hard-margin (or Hard-margin SVM),
\begin{align*}
\min_{\w}   & \ \ \frac{1}{2} \w^{\top}\w \\
\mbox{s.t.} & \ \ y_i \w^{\top} \phi(\x_i) \geq 1, \ i = 1, \ldots, m,
\end{align*}
is regarded as the maximization of the minimum margin \{$\min \{\gamma_i\}_{i=1}^{m}$\}.

In non-separable cases where the training examples cannot be separated with the zero error, SVM with soft-margin (or Soft-margin SVM) is posed,
\begin{align} \label{eq:soft-margin SVM}
\begin{split}
\min_{\w, \xivec}  & \ \ \frac{1}{2} \w^{\top}\w + C \sum_{i=1}^{m}\xi_i \\
\mbox{s.t.}        & \ \ y_i \w^{\top} \phi(\x_i) \geq 1 - \xi_i,\\ & \ \ \xi_i \geq 0, \; i = 1, \ldots, m.
\end{split}
\end{align}
where $\xivec = [\xi_1, \ldots, \xi_m]^\top$ measure the losses of instances, and $C$ is a trading-off parameter. There exists a constant $\bar{C}$ such that (\ref{eq:soft-margin SVM}) can be equivalently reformulated as,
\begin{align*}
\max_{\w} & \ \ \gamma_0 - \bar{C} \sum\nolimits_{i=1}^{m} \xi_i \\
\mbox{s.t.} & \ \ \gamma_i \geq \gamma_0 - \xi_i,\\ & \ \ \xi_i \geq 0, \; i = 1, \ldots, m,
\end{align*}
where $\gamma_0$ is a relaxed minimum margin, and $\bar{C}$ is the trading-off parameter. Note that $\gamma_0$ indeed characterizes the top-$p$ minimum margin \cite{Gao2012On}; hence, SVMs (with both hard-margin and soft-margin) consider only a single-point margin and have not exploited the whole margin distribution.

\section{LDM} \label{sec:LDM}

In this section, we first formulate the margin distribution, and then present the optimization algorithms and the theoretical guarantee.

\subsection{Formulation} \label{sec:formulation}

The two most straightforward statistics for characterizing the margin distribution are the first- and second-order statistics, that is, the mean and the variance of the margin. Formally, denote $\X$ as the matrix whose $i$-th column is $\phi(\x_i)$, i.e., $\X = [\phi(\x_1) \ldots \phi(\x_m)]$, $\y = [y_1, \ldots, y_m]^\top$ is a column vector, and $\Y$ is a $m \times m$ diagonal matrix with $y_1, \ldots, y_m$ as the diagonal elements. According to the definition in (\ref{defn:margin}), the margin mean is
\begin{align} \label{eq:margin mean}
\bar{\gamma} = \frac{1}{m}\sum_{i=1}^{m} y_i \w^\top \phi(\x_i) = \frac{1}{m} (\X \y)^\top \w,
\end{align}
and the margin variance is
\begin{align} \label{eq:margin variance}
\begin{split}
\hat{\gamma} & = \sum_{i=1}^m \sum_{j=1}^m (y_i \w^\top \phi(\x_i) - y_j \w^\top \phi(\x_j))^2    \\
   & = \frac{2}{m^2} (m \w^\top \X \X^\top \w - \w^\top \X \y \y^\top \X^\top \w).
\end{split}
\end{align}
Inspired by the recent theoretical result \cite{Gao2012On}, LDM attempts to maximize the margin mean and minimize the margin variance simultaneously.

We first consider a simpler scenario, i.e., the separable cases where the training examples can be separated with the zero error. In these cases, the maximization of the margin mean and the minimization of the margin variance leads to the following hard-margin LDM,
\begin{align*}
\min_{\w}   & \ \ \frac{1}{2} \w^\top \w + \lambda_1 \hat{\gamma} - \lambda_2 \bar{\gamma} \\
\mbox{s.t.} & \ \ y_i \w^{\top} \phi(\x_i)  \geq 1, \ i = 1, \ldots, m,
\end{align*}
where $\lambda_1$ and $\lambda_2$ are the parameters for trading-off the margin variance, the margin mean and the model complexity. It's evident that the hard-margin LDM subsumes the hard-margin SVM when $\lambda_1$ and $\lambda_2$ equal $0$.

For the non-separable cases, similar to soft-margin SVM, the soft-margin LDM leads to
\begin{align} \label{eq:soft-margin LDM}
\begin{split}
\min_{\w, \xivec} & \ \ \frac{1}{2} \w^\top \w + \lambda_1 \hat{\gamma} - \lambda_2 \bar{\gamma} + C \sum_{i=1}^{m} \xi_i \\
\mbox{s.t.}  & \ \ y_i \w^{\top} \phi(\x_i) \geq 1 - \xi_i,\\ & \ \ \xi_i \geq 0, \ i = 1, \ldots, m.
\end{split}
\end{align}
Similarly, soft-margin LDM subsumes the soft-margin SVM if $\lambda_1$ and $\lambda_2$ both equal $0$. Because the soft-margin SVM often performs much better than the hard-margin one, in the following we will focus on soft-margin LDM and if without clarification, LDM is referred to the soft-margin LDM.

\subsection{Optimization} \label{sec:optimization}

We in this section first present a dual coordinate descent method for kernel LDM, and then present an average stochastic gradient descent (ASGD) method for large scale linear kernel LDM.

\subsubsection{Kernel LDM} \label{sec:kernelLDM}

By substituting (\ref{eq:margin mean})-(\ref{eq:margin variance}), (\ref{eq:soft-margin LDM}) leads to the following quadratic programming problem,
\begin{align} \label{eq:primal-LDM}
\begin{split}
\min_{\w, \xivec} & \ \ \frac{1}{2} \w^\top \w + \frac{2\lambda_1}{m^2} (m \w^\top \X \X^\top \w - \w^\top \X \y \y^\top \X^\top \w) \\
                  & \ \ - \lambda_2 \frac{1}{m} (\X \y)^\top \w + C \sum_{i=1}^{m} \xi_i \\
\mbox{s.t.}       & \ \ y_i \w^{\top} \phi(\x_i) \geq 1 - \xi_i, \\ & \ \ \xi_i \geq 0, \ i = 1, \ldots, m.
\end{split}
\end{align}
(\ref{eq:primal-LDM}) is often intractable due to the high or infinite dimensionality of $\phi(\cdot)$. Fortunately, inspired by the representer theorem in \cite{Scholkopf2001learning}, the following theorem states that the optimal solution for (\ref{eq:primal-LDM}) can be spanned by $\{\phi(\x_i), 1 \leq i \leq m\}$.
\begin{theorem} \label{thm:representer theorem}
The optimal solution $\w^*$ for problem (\ref{eq:primal-LDM}) admits a representation of the form
\begin{align} \label{eq:optimal solution form}
\w^* = \sum_{i=1}^m \alpha_i \phi(\x_i) = \X \alphavec,
\end{align}
where $\alphavec = [\alpha_1, \dots, \alpha_m]^\top$ are the coefficients.
\end{theorem}

\begin{proof}
$\w^*$ can be decomposed into a part that lives in the span of $\phi(\x_i)$ and an orthogonal part, i.e.,
\begin{align*}
\w = \sum_{i=1}^m \alpha_i \phi(\x_i) + \v = \X \alphavec + \v
\end{align*}
for some $\alphavec = [\alpha_1,\ldots,\alpha_m]^\top$ and $\v$ satisfying $\phi(\x_j)^\top \v = 0$ for all $j$, i.e., $\X^\top \v = \0 $. Note that
\begin{align*}
\X^\top \w = \X^\top (\X \alphavec + \v) = \X^\top \X \alphavec,
\end{align*}
so the second and the third terms of (\ref{eq:primal-LDM}) are independent of $\v$; further note that the constraint is also independent of $\v$, thus the last terms of (\ref{eq:primal-LDM}) is also independent of $\v$.

As for the first term of (\ref{eq:primal-LDM}), since $\X^\top \v = \0 $, consequently we get
\begin{align*}
\w^\top \w & = (\X \alphavec + \v)^\top (\X \alphavec + \v) = \alphavec^\top \X^\top \X \alphavec + \v^\top \v  \\
& \geq \alphavec^\top \X^\top \X \alphavec
\end{align*}
with equality occurring if and only if $\v = \0 $.

So, setting $\v = \0 $ does not affect the second, the third and the last term while strictly reduces the first term of (\ref{eq:primal-LDM}). Hence, $\w^*$ for problem (\ref{eq:primal-LDM}) admits a representation of the form ({\ref{eq:optimal solution form}}).
\end{proof}

According to Theorem \ref{thm:representer theorem}, we have
\begin{align*}
\X^\top \w & = \X^\top \X \alphavec = \G \alphavec,    \\
\w^\top \w & = \alphavec^\top \X^\top \X \alphavec = \alphavec^\top \G \alphavec,
\end{align*}
where $\G = \X^\top \X$ is the kernel matrix. Let $\G_{:i}$ denote the $i$-th column of $\G$, then (\ref{eq:primal-LDM}) can be cast as
\begin{align} \label{eq:kernel soft-margin LDM}
\begin{split}
\min_{\alphavec,\xivec} & \ \ \frac{1}{2} \alphavec^\top \Q \alphavec + \p^\top \alphavec + C \sum_{i=1}^{m} \xi_i  \\
\mbox{s.t.}                    & \ \ y_i \alphavec^{\top} \G_{:i} \geq 1 - \xi_i,\\ & \ \ \xi_i \geq 0, \ i = 1, \ldots, m, \\
\end{split}
\end{align}
where $\Q = 4\lambda_1 (m\G^\top\G - (\G\y) (\G\y)^\top)/m^2 + \G$ and $\p = -\lambda_2 \G \y/m$. By introducing the lagrange multipliers $\betavec = [\beta_1, \ldots, \beta_m]^\top$ and $\etavec = [\eta_1, \ldots, \eta_m]^\top$ for the first and the second constraints respectively, the Lagrangian of (\ref{eq:kernel soft-margin LDM}) leads to
\begin{align} \label{eq:lagrange}
\begin{split}
L(\alphavec, \xivec, \betavec, \etavec) & = \frac{1}{2} \alphavec^\top \Q \alphavec + \p^\top \alphavec + C \sum_{i=1}^{m} \xi_i \\
& \ \ \ - \sum_{i=1}^{m} \beta_i (y_i \alphavec^{\top} \G_{:i} - 1 + \xi_i) - \sum_{i=1}^{m} \eta_i \xi_i.
\end{split}
\end{align}
By setting the partial derivations of \{$\alphavec, \xivec$\} to zero, we have
\begin{align}
\label{eq:KKT1}
\frac{\partial L}{\partial \alphavec} & = \Q \alphavec + \p - \sum_{i=1}^{m} \beta_i y_i \G_{:i},   \\
\label{eq:KKT2}
\frac{\partial L}{\partial \xi_i} & = C - \beta_i - \eta_i = 0, \ i = 1, \dots, m.
\end{align}
By substituting (\ref{eq:KKT1}) and (\ref{eq:KKT2}) into (\ref{eq:lagrange}), the dual \footnote{Here we omit constants without influence on optimization.}
of (\ref{eq:kernel soft-margin LDM}) can be cast as:
\begin{align} \label{eq:final LDM}
\begin{split}
\min_{\betavec} & \ \ f(\betavec) =  \frac{1}{2} \betavec^\top \H \betavec + \left( \frac{\lambda_2}{m} \H \e - \e \right)^\top \betavec, \\
\mbox{s.t.}     & \ \ 0 \leq \beta_i \leq C, \ i = 1, \ldots, m.
\end{split}
\end{align}
where $\H = \Y \G \Q^{-1} \G \Y $, $\Q^{-1}$ refers to the inverse matrix of $\Q$ and $\e$ stands for the all-one vector. Due to the simple decoupled box constraint and the convex quadratic objective function, as suggested by \cite{Yuan2012Recent}, (\ref{eq:final LDM}) can be efficiently solved by the dual coordinate descent method. In dual coordinate descent method \cite{Hsieh2008A}, one of the variables is selected to minimize while the other variables are kept as constants at each iteration, and a closed-form solution can be achieved at each iteration. Specifically, to minimize $\beta_i$ by keeping the other $\beta_{j\neq i}$'s as constants, one needs to solve the following subproblem,
\begin{align} \label{eq:simpleQP}
\begin{split}
\min_{t}    & \ \ f(\betavec + t \e_i) \\
\mbox{s.t.} & \ \ 0 \leq \beta_i + t \leq C,
\end{split}
\end{align}
where $\e_i$ denotes the vector with $1$ in the $i$-th coordinate and $0$'s elsewhere. Let $\H = [h_{ij}]_{i,j=1,\dots,m}$, we have
\begin{align*}
f(\betavec + t \e_i) = \frac{1}{2} h_{ii} t^2 + [\nabla f(\betavec)]_i t + f(\betavec),
\end{align*}
where $[\nabla f(\betavec)]_i$ is the $i$-th component of the gradient $\nabla f(\betavec)$. Note that $f(\betavec)$ is independent of $t$ and thus can be dropped. Considering that $f(\betavec + t \e_i)$ is a simple quadratic function of $t$, and further note the box constraint $0 \leq \alpha_i \leq C$, the minimizer of (\ref{eq:simpleQP}) leads to a closed-form solution,
\begin{align*}
\beta_i^{new} = \min\left( \max \left( \beta_i - \frac{[\nabla f(\betavec)]_i}{h_{ii}}, 0 \right), C \right).
\end{align*}
Algorithm \ref{alg:kernelLDM} summarizes the pseudo-code of kernel LDM.

\begin{algorithm}[!t]
\caption{Kernel LDM}
\begin{algorithmic}
\STATE {\bfseries Input:} Data set $\X$, $\lambda_1$, $\lambda_2$, $C$
\STATE {\bfseries Output:} $\alphavec$
\STATE  Initialize $\betavec = \0 $, $\alphavec = \frac{\lambda_2}{m} \Q^{-1} \G \y$, $\A = \Q^{-1} \G \Y$, $h_{ii} = \e_i^\top \Y \G \Q^{-1} \G \Y \e_i$;
\WHILE{$\betavec$ not converge}
\FOR{$i = 1, \dots m$}
\STATE  $[\nabla f(\betavec)]_i \leftarrow \e_i^\top \Y \G \alphavec - 1$;
\STATE  $\beta_i^{old} \leftarrow \beta_i$;
\STATE  $\beta_i \leftarrow \min\left( \max \left( \beta_i -\frac{[\nabla f(\betavec)]_i}{h_{ii}}, 0 \right), C \right)$;
\STATE  $\alphavec \leftarrow \alphavec + (\beta_i - \beta_i^{old}) \A \e_i$;
\ENDFOR
\ENDWHILE
\end{algorithmic}
\label{alg:kernelLDM}
\end{algorithm}

For prediction, according to (\ref{eq:KKT1}), one can obtain the coefficients $\alphavec$ from the optimal $\betavec^*$ as
\begin{align*}
\alphavec & = \Q^{-1} (\G \Y \betavec^* - \p) = \Q^{-1} \left( \frac{\lambda_2}{m} \G \Y \e + \G \Y \betavec^* \right)  \\
          & = \Q^{-1} \G \Y \left( \frac{\lambda_2}{m} \e + \betavec^* \right).
\end{align*}
Hence for testing instance $\z$, its label can be obtained by
\begin{align*}
sgn\left( \w^\top \phi(\z) \right) = sgn\left( \sum_{i=1}^{m}\alpha_i k(\x_i, \z) \right).
\end{align*}

\subsubsection{Large Scale Kernel LDM} \label{sec:linearLDM}

In section \ref{sec:kernelLDM}, the proposed method can efficiently deal with kernel LDM. However, the inherent computational cost for the kernel matrix in kernel LDM takes $O(m^2)$ time, which might be computational prohibitive for large scale problems. To make LDM more useful, in the following, we present a fast linear kernel LDM for large scale problems by adopting the average stochastic gradient descent (ASGD) method \cite{Polyak1992Acceleration}.

For linear kernel LDM, (\ref{eq:soft-margin LDM}) can be reformulated as the following form,
\begin{align} \label{eq:linear LDM}
\begin{split}
\min_{\w} \ \ & g(\w) = \frac{1}{2} \w^\top \w + \frac{2 \lambda_1}{m^2} \w^\top (m  \X \X^\top - \X \y \y^\top \X^\top ) \w - \frac{\lambda_2}{m} (\X \y)^\top \w \\ & + C \sum_{i=1}^m \max\{0, 1 - y_i \w^\top \x_i\},
\end{split}
\end{align}
where $\X = [\x_1 \ldots \x_m]$, $\y = [y_1, \ldots, y_m]^\top$ is a column vector.

For large scale problems, computing the gradient of (\ref{eq:linear LDM}) is expensive because its computation involves all the training examples. Stochastic gradient descent (SGD) works by computing a noisy unbiased estimation of the gradient via sampling a subset of the training examples. Theoretically, when the objective is convex, it can be shown that in expectation, SGD converges to the global optimal solution \cite{Kushner2003Stochastic,Bottou2010Large-Scale}. During the past decade, SGD has been applied to various machine learning problems and achieved promising performances \cite{Zhang2004Solving,Shalev-Shwartz2007Pegasos,Bordes2009SGD-QN,Shamir2013Stochastic}.

The following theorem presents an approach to obtain an unbiased estimation of the gradient $\nabla g(\w)$.
\begin{theorem} \label{thm:SGD4LDM}
If two examples $(\x_i, y_i)$ and $(\x_j, y_j)$ are sampled from training set randomly, then
\begin{align} \label{eq:SGD update for LDM}
\nabla g(\w, & \x_i, \x_j) = 4 \lambda_1 \x_i \x_i^\top \w - 4 \lambda_1 y_i \x_i y_j \x_j ^\top \w + \w
- \lambda_2 y_i \x_i - m C
\begin{cases}
y_i \x_i & i \in \I ,    \\
0 & otherwise,
\end{cases}
\end{align}
where $\I \equiv \{ i \ | \ y_i \w^\top \x_i < 1 \}$ is an unbiased estimation of $\nabla g(\w)$.
\end{theorem}

\begin{proof}
Note that the gradient of $g(\w)$ is
\begin{align*}
\nabla g(\w) = \Q \w + \p - C \sum_{i=1}^m y_i \x_i, i \in \I,
\end{align*}
where $\Q = 4\lambda_1 (m\X\X^\top - \X\y (\X\y)^\top)/m^2 + \I$ and $\p = -\lambda_2 \X \y/m$. Further note that
\begin{align} \label{eq:expectation}
\begin{split}
& E_{\x_i} [\x_i \x_i^\top] = \frac{1}{m} \sum_{i=1}^m \x_i \x_i^\top = \frac{1}{m} \X \X^\top,   \\
& E_{\x_i} [y_i \x_i] = \frac{1}{m} \sum_{i=1}^m y_i \x_i = \frac{1}{m} \X \y.
\end{split}
\end{align}
According to the linearity of expectation, the independence between $\x_i$ and $\x_j$, and with (\ref{eq:expectation}), we have
\begin{align*}
& \ \ \ \ E_{\x_i, \x_j} [\nabla g(\w, \x_i, \x_j)] \\
& = 4 \lambda_1 E_{\x_i} [\x_i \x_i^\top] \w - 4 \lambda_1 E_{\x_i} [y_i \x_i] E_{\x_j} [y_j \x_j]^\top \w + \w - \lambda_2 E_{\x_i} [y_i \x_i] - mC E_{\x_i} [y_i \x_i \ | \ i \in \I] \\
& = 4 \lambda_1 \frac{1}{m} \X \X^\top \w - 4 \lambda_1 \frac{1}{m} \X \y \left( \frac{1}{m} \X \y \right)^\top \w + \w - \lambda_2 \frac{1}{m} \X \y - mC \frac{1}{m} \sum_{i=1}^m y_i \x_i, i \in \I \\
& = \Q \w + \p - C \sum_{i=1}^m y_i \x_i, i \in \I  \\
& = \nabla g(\w).
\end{align*}
It is shown that $\nabla g(\w, \x_i, \x_j)$ is a noisy unbiased gradient of $g(\w)$. 
\end{proof}

With Theorem \ref{thm:SGD4LDM}, the stochastic gradient update can be formed as
\begin{align} \label{eq:SGD update}
\w_{t+1} = \w_t - \eta_t \nabla g(\w, \x_i, \x_j),
\end{align}
where \(\eta_t\) is a suitably chosen step-size parameter in the \(t\)-th iteration.

In practice, we use averaged stochastic gradient descent (ASGD) which is more robust than SGD \cite{Xu2010Towards}. At each iteration, besides performing the normal stochastic gradient update (\ref{eq:SGD update}), we also compute
\begin{align*}
\wbar_t = \frac{1}{t - t_0} \sum_{i=t_0+1}^t \w_i,
\end{align*}
where $t_0$ determines when we engage the averaging process. This average can be computed efficiently using a recursive formula:
\begin{align*}
\wbar_{t+1} = \wbar_t + \mu_t (\w_{t+1} - \wbar_t),
\end{align*}
where $\mu_t = 1 / \max \{ 1, t - t_0 \}$.

Algorithm \ref{alg:linearLDM} summarizes the pseudo-code of large scale kernel LDM.

\begin{algorithm}[!t]
\caption{Large Scale Kernel LDM}
\begin{algorithmic}
\STATE {\bfseries Input:} Data set $\X$, $\lambda_1$, $\lambda_2$, $C$
\STATE {\bfseries Output:} $\wbar$
\STATE  Initialize $\u = \0 $, $T = 5$;
\FOR{$t = 1, \dots Tm$}
\STATE  Randomly sample two training examples $(\x_i, y_i)$ and $(\x_j, y_j)$;
\STATE  Compute $\nabla g(\w, \x_i, \x_j)$ as in (\ref{eq:SGD update for LDM});
\STATE  $\w \leftarrow \w - \eta_t \nabla g(\w, \x_i, \x_j)$;
\STATE  $\wbar \leftarrow \wbar + \mu_t (\w - \wbar)$;
\ENDFOR
\end{algorithmic}
\label{alg:linearLDM}
\end{algorithm}

\subsection{Analysis} \label{sec:analysis}

In this section, we study the statistical property of LDM. Specifically, we derive a bound on the expectation of error for LDM according to the leave-one-out cross-validation estimate, which is an unbiased estimate of the probability of test error.

Here we only consider the linear case (\ref{eq:linear LDM}) for simplicity, however, the results are also applicable to any other feature mapping $\phi$. Following the same steps in Section \ref{sec:kernelLDM}, one can have the dual problem of (\ref{eq:linear LDM}), i.e.,
\begin{align} \label{eq:dual linear LDM}
\begin{split}
\min_{\alphavec} & \ \ f(\alphavec) =  \frac{1}{2} \alphavec^\top \H \alphavec + \left( \frac{\lambda_2}{m} \H \e - \e \right)^\top \alphavec, \\
\mbox{s.t.}     & \ \ 0 \leq \alpha_i \leq C, \ i = 1, \ldots, m.
\end{split}
\end{align}
where $\H = \Y \X^\top \Q^{-1} \X \Y $, $\Q = \frac{4\lambda_1}{m^2} (m \X \X^\top - \X\y (\X\y)^\top) + \I$, $\e$ stands for the all-one vector and $\Q^{-1}$ refers to the inverse matrix of $\Q$.

\begin{theorem} \label{thm:bound}
Let $\alphavec$ denote the optimal solution of (\ref{eq:dual linear LDM}), and $E[R(\alphavec)]$ be the expectation of the probability of error, then we have
\begin{align} \label{eq:bound}
E[R(\alphavec)] \leq \frac{E[h \sum_{i \in \I_1} \alpha_i + |\I_2|]}{m},
\end{align}
where $\I_1 \equiv \{ i \ | \ 0 < \alpha_i < C \}$, $\I_2 \equiv \{ i \ | \ \alpha_i = C \}$ and $h = \max\{ diag\{\H\} \}$.
\end{theorem}

\begin{proof}
Suppose
\begin{align} \label{eq:solution}
\begin{split}
\alphavec^* & = \argmin_{0 \leq \alphavec \leq C} f(\alphavec), \\
\alphavec^i & = \argmin_{0 \leq \alphavec \leq C, \alpha_i = 0} f(\alphavec), \ \ i = 1, \ldots, m,
\end{split}
\end{align}
and the corresponding solution for the linear kernel LDM are $\w^*$ and $\w^i$, respectively.

As shown in \cite{Luntz1969On},
\begin{align} \label{eq:leave-one-out}
E[R(\alphavec)] = \frac{E[\Lcal((\x_1,y_1), \ldots, (\x_m, y_m))]}{m},
\end{align}
where $\Lcal((\x_1,y_1), \ldots, (\x_m, y_m))$ is the number of errors in the leave-one-out procedure. Note that if $\alpha_i^* = 0$, $(\x_i, y_i)$ will always be classified correctly in the leave-one-out procedure according to the KKT conditions. So for any misclassified example $(\x_i, y_i)$, we only need to consider the following two cases:

1) $0 < \alpha_i^* < C$, according to the definition in (\ref{eq:solution}), we have
\begin{align} \label{eq:relation1}
f(\alphavec^i) - \min_t f(\alphavec^i + t \e_i) \leq f(\alphavec^i) - f(\alphavec^*) \leq f(\alphavec^* - \alpha_i^* \e_i) - f(\alphavec^*),
\end{align}
where $\e_i$ denotes a vector with $1$ in the $i$-th coordinate and $0$'s elsewhere. We can find that, the left-hand side of (\ref{eq:relation1}) is equal to $(1 - y_i \x_i^\top \w^i)^2/2 h_{ii}$, and the right-hand side of (\ref{eq:relation1}) is equal to ${\alpha_i^*}^2 h_{ii}/2$. So we have
\begin{align*}
(1 - y_i \x_i^\top \w^i)^2/2 h_{ii} \leq {\alpha_i^*}^2 h_{ii}/2.
\end{align*}
Further note that $y_i \x_i^\top \w^i < 0$, rearranging the above we can obtain $1 \leq \alpha_i^* h_{ii}$.

2) $\alpha_i^* = C$, all these examples will be misclassified in the leave-one-out procedure.

So we have
\begin{align*}
\Lcal((\x_1,y_1), \ldots, (\x_m, y_m)) \leq h \sum_{i \in \I_1} \alpha_i^* + |\I_2|,
\end{align*}
where $\I_1 \equiv \{ i \ | \ 0 < \alpha_i^* < C \}$, $\I_2 \equiv \{ i \ | \ \alpha_i^* = C \}$ and $h = \max\{ h_{ii}, i = 1, \ldots, m \}$. Take expectation on both side and with (\ref{eq:leave-one-out}), we get that (\ref{eq:bound}) holds.
\end{proof}

\begin{table*}[!t]
\scriptsize
\caption{Characteristics of experimental data sets.} \vspace{+2mm}
\centering
\scalebox{1.15}[1.15]{
\begin{tabular}{ c | l c c | l c c}
\hline
\hline \noalign{\smallskip}
 Scale & Dataset & \#Instance & \#Feature & Dataset & \#Instance & \#Feature \\
\noalign{\smallskip} \hline \noalign{\smallskip}
\emph{regular} &\emph{promoters} & {\myfont 106} & {\myfont 57} & \emph{haberman} & {\myfont 306} & {\myfont 14} \\
 & \emph{planning} & {\myfont 182} & {\myfont 12} & \emph{vehicle} & {\myfont 435} & {\myfont 16} \\
 & \emph{colic} & {\myfont 188} & {\myfont 13} & \emph{clean1} & {\myfont 476} & {\myfont 166} \\
 & \emph{parkinsons} & {\myfont 195} & {\myfont 22} & \emph{wdbc} & {\myfont 569} & {\myfont 14} \\
 & \emph{colic.ORIG} & {\myfont 205} & {\myfont 17} & \emph{isolet} & {\myfont 600} & {\myfont 51} \\
 & \emph{sonar} & {\myfont 208} & {\myfont 60} & \emph{credit-a} & {\myfont 653} & {\myfont 15} \\
 & \emph{vote} & {\myfont 232} & {\myfont 16} & \emph{austra} & {\myfont 690} & {\myfont 15} \\
 & \emph{house} & {\myfont 232} & {\myfont 16} & \emph{australian} & {\myfont 690} & {\myfont 42} \\
 & \emph{heart} & {\myfont 270} & {\myfont 9} & \emph{fourclass} & {\myfont 862} & {\myfont 2} \\
 & \emph{breast} & {\myfont 277} & {\myfont 9} & \emph{german} & {\myfont 1,000} & {\myfont 59} \\
\noalign{\smallskip} \hline \noalign{\smallskip}
\emph{large} & \emph{farm-ads} & {\myfont 4,143} & {\myfont 54,877} & \emph{ijcnn1} & {\myfont 141,691} & {\myfont 22} \\
 & \emph{news20} & {\myfont 19,996} & {\myfont 1,355,191} & \emph{skin} & {\myfont 245,057} & {\myfont 3} \\
 & \emph{adult-a} & {\myfont 32,561} & {\myfont 123} & \emph{covtype} & {\myfont 581,012} & {\myfont 54} \\
 & \emph{w8a} & {\myfont 49,749} & {\myfont 300} & \emph{rcv1} & {\myfont 697,641} & {\myfont 47,236} \\
 & \emph{cod-rna} & {\myfont 59,535} & {\myfont 8} & \emph{url} & {\myfont 2,396,130} & {\myfont 3,231,961} \\
 & \emph{real-sim} & {\myfont 72,309} & {\myfont 20,958} & \emph{kdd2010} & {\myfont 8,407,752} & {\myfont 20,216,830} \\
\noalign{\smallskip} \hline
\hline
\end{tabular}}
\label{table 1}
\end{table*}

\section{Empirical Study} \label{sec:experiments}

In this section, we empirically evaluate the effectiveness of LDM on a broad range of data sets. We first introduce the experimental settings in Section \ref{sec:experimental setup}, and then compare LDM with SVM and three state-of-the-art approaches\footnote{These approaches will be briefly introduced in Section \ref{sec:Related Work}.} in Section \ref{sec:Resultsb} and Section \ref{sec:Resultsl}. In addition, we also study the cumulative margin distribution produced by LDM and SVM in Section \ref{sec:IMD}. The computational cost and parameter influence are presented in Section \ref{sec:tc} and Section \ref{sec:pi}, respectively.

\begin{table*}[!t]
\scriptsize
\caption{Accuracy (mean$\pm$std.) comparison on regular scale data sets. Linear kernels are used. The best accuracy on each data set is bolded. $\bullet$/$\circ$ indicates the performance is significantly better/worse than SVM (paired $t$-tests at 95\% significance level). The win/tie/loss counts are summarized in the last row.}
\centering
\vskip 0.1in
\scalebox{1.15}[1.15]{
\begin{tabular}{ l | l l l l l }
\hline \hline \noalign{\smallskip}
Dataset & \multicolumn{1}{c}{SVM} & \multicolumn{1}{c}{MDO} & \multicolumn{1}{c}{MAMC} & \multicolumn{1}{c}{KM-OMD} & \multicolumn{1}{c}{LDM} \\
\noalign{\smallskip} \hline \noalign{\smallskip}
\emph{promoters} & {\myfont 0.723$\pm$0.071 } & {\myfont 0.713$\pm$0.067} & {\myfont 0.520$\pm$0.096}$\circ$ & {\myfontb 0.736$\pm$0.061} & {\myfont 0.721$\pm$0.069} \\
\emph{planning-relax} & {\myfont 0.683$\pm$0.031 } & {\myfont 0.605$\pm$0.185}$\circ$ & {\myfontb 0.706$\pm$0.034}$\bullet$ & {\myfont 0.479$\pm$0.050}$\circ$ & {\myfontb 0.706$\pm$0.034}$\bullet$ \\
\emph{colic} & {\myfont 0.814$\pm$0.035 } & {\myfont 0.781$\pm$0.154} & {\myfont 0.661$\pm$0.062}$\circ$ & {\myfont 0.813$\pm$0.028} & {\myfontb 0.832$\pm$0.026}$\bullet$ \\
\emph{parkinsons} & {\myfont 0.846$\pm$0.038 } & {\myfont 0.732$\pm$0.270}$\circ$ & {\myfont 0.764$\pm$0.035}$\circ$ & {\myfont 0.814$\pm$0.024}$\circ$ & {\myfontb 0.865$\pm$0.030}$\bullet$ \\
\emph{colic.ORIG} & {\myfont 0.618$\pm$0.027 } & {\myfont 0.624$\pm$0.040} & {\myfont 0.623$\pm$0.027} & {\myfontb 0.635$\pm$0.045}$\bullet$ & {\myfont 0.619$\pm$0.042} \\
\emph{sonar} & {\myfont 0.725$\pm$0.039 } & {\myfont 0.734$\pm$0.035} & {\myfont 0.533$\pm$0.045}$\circ$ & {\myfontb 0.766$\pm$0.033}$\bullet$ & {\myfont 0.736$\pm$0.036} \\
\emph{vote} & {\myfont 0.934$\pm$0.022 } & {\myfont 0.587$\pm$0.435}$\circ$ & {\myfont 0.884$\pm$0.022}$\circ$ & {\myfont 0.957$\pm$0.013}$\bullet$ & {\myfontb 0.970$\pm$0.014}$\bullet$ \\
\emph{house} & {\myfont 0.942$\pm$0.015 } & {\myfont 0.943$\pm$0.015} & {\myfont 0.883$\pm$0.029}$\circ$ & {\myfont 0.957$\pm$0.020}$\bullet$ & {\myfontb 0.968$\pm$0.011}$\bullet$ \\
\emph{heart} & {\myfont 0.799$\pm$0.029 } & {\myfont 0.826$\pm$0.026}$\bullet$ & {\myfont 0.537$\pm$0.057}$\circ$ & {\myfontb 0.836$\pm$0.026}$\bullet$ & {\myfont 0.791$\pm$0.030} \\
\emph{breast-cancer} & {\myfont 0.717$\pm$0.033 } & {\myfont 0.710$\pm$0.031} & {\myfont 0.706$\pm$0.027} & {\myfont 0.696$\pm$0.031}$\circ$ & {\myfontb 0.725$\pm$0.027}$\bullet$ \\
\emph{haberman} & {\myfont 0.734$\pm$0.030 } & {\myfont 0.728$\pm$0.029} & {\myfontb 0.738$\pm$0.020} & {\myfont 0.667$\pm$0.040}$\circ$ & {\myfontb 0.738$\pm$0.020} \\
\emph{vehicle} & {\myfont 0.959$\pm$0.012 } & {\myfont 0.956$\pm$0.012} & {\myfont 0.566$\pm$0.160}$\circ$ & {\myfontb 0.960$\pm$0.010} & {\myfont 0.959$\pm$0.013} \\
\emph{clean1} & {\myfont 0.803$\pm$0.035 } & {\myfont 0.798$\pm$0.031} & {\myfont 0.561$\pm$0.025}$\circ$ & {\myfontb 0.821$\pm$0.027}$\bullet$ & {\myfont 0.814$\pm$0.019}$\bullet$ \\
\emph{wdbc} & {\myfont 0.963$\pm$0.012 } & {\myfont 0.966$\pm$0.010} & {\myfont 0.623$\pm$0.020}$\circ$ & {\myfont 0.968$\pm$0.009}$\bullet$ & {\myfontb 0.968$\pm$0.011}$\bullet$ \\
\emph{isolet} & {\myfont 0.995$\pm$0.003 } & {\myfont 0.501$\pm$0.503}$\circ$ & {\myfont 0.621$\pm$0.207}$\circ$ & {\myfont 0.995$\pm$0.003} & {\myfontb 0.997$\pm$0.002}$\bullet$ \\
\emph{credit-a} & {\myfont 0.861$\pm$0.014 } & {\myfont 0.862$\pm$0.013} & {\myfont 0.596$\pm$0.063}$\circ$ & {\myfont 0.863$\pm$0.013} & {\myfontb 0.864$\pm$0.013}$\bullet$ \\
\emph{austra} & {\myfont 0.857$\pm$0.013 } & {\myfont 0.842$\pm$0.055} & {\myfont 0.567$\pm$0.044}$\circ$ & {\myfont 0.858$\pm$0.013} & {\myfontb 0.859$\pm$0.015} \\
\emph{australian} & {\myfont 0.844$\pm$0.019 } & {\myfont 0.842$\pm$0.020} & {\myfont 0.576$\pm$0.049}$\circ$ & {\myfont 0.858$\pm$0.016}$\bullet$ & {\myfontb 0.866$\pm$0.014}$\bullet$ \\
\emph{fourclass} & {\myfont 0.724$\pm$0.014 } & {\myfont 0.377$\pm$0.238}$\circ$ & {\myfont 0.641$\pm$0.020}$\circ$ & {\myfontb 0.736$\pm$0.014}$\bullet$ & {\myfont 0.723$\pm$0.014} \\
\emph{german} & {\myfont 0.711$\pm$0.030 } & {\myfont 0.737$\pm$0.014}$\bullet$ & {\myfont 0.697$\pm$0.017}$\circ$ & {\myfont 0.729$\pm$0.017}$\bullet$ & {\myfontb 0.738$\pm$0.016}$\bullet$ \\
\noalign{\smallskip} \hline \noalign{\smallskip}
Ave. accuracy & \multicolumn{1}{c}{ \myfont{ 0.813 }} & \multicolumn{1}{c}{ \myfont{ 0.743 }} & \multicolumn{1}{c}{ \myfont{ 0.650 }} & \multicolumn{1}{c}{ \myfont{ 0.807 }} & \multicolumn{1}{c}{ \myfont{ 0.823 }} \\
\noalign{\smallskip} \hline \noalign{\smallskip}
LDM: w/t/l & \multicolumn{1}{c}{ \myfont{ 12/8/0}} & \multicolumn{1}{c}{ \myfont{ 9/10/1}} & \multicolumn{1}{c}{ \myfont{ 17/3/0}} & \multicolumn{1}{c}{ \myfont{ 10/5/5}} \\
\noalign{\smallskip} \hline \hline
\end{tabular}}
\label{table 2}
\end{table*}

\subsection{Experimental Setup} \label{sec:experimental setup}

We evaluate the effectiveness of our proposed LDMs on twenty regular scale data sets and twelve large scale data sets, including both UCI data sets and real-world data sets like KDD2010\footnote{https://pslcdatashop.web.cmu.edu/KDDCup/downloads.jsp}. Table \ref{table 1} summarizes the statistics of these data sets. The data set size is ranged from 106 to more than 8,000,000, and the dimensionality is ranged from 2 to more than 20,000,000, covering a broad range of properties. All features are normalized into the interval $[0, 1]$. For each data set, half of examples are randomly selected as the training data, and the remaining examples are used as the testing data. For regular scale data sets, both linear and RBF kernels are evaluated. Experiments are repeated for 30 times with random data partitions, and the average accuracies as well as the standard deviations are recorded. For large scale data sets, linear kernel is evaluated. Experiments are repeated for 10 times with random data partitions, and the average accuracies (with standard deviations) are recorded.

\begin{table*}[!t]
\scriptsize
\caption{Accuracy (mean$\pm$std.) comparison on regular scale data sets. RBF kernels are used. The best accuracy on each data set is bolded. $\bullet$/$\circ$ indicates the performance is significantly better/worse than SVM (paired $t$-tests at 95\% significance level). The win/tie/loss counts are summarized in the last row. MDO does not have results since it is specified for the linear kernel.}
\centering
\vskip 0.1in
\scalebox{1.15}[1.15]{
\begin{tabular}{ l | l c l l l }
\hline \hline \noalign{\smallskip}
Dataset & \multicolumn{1}{c}{SVM} & \multicolumn{1}{c}{MDO} & \multicolumn{1}{c}{MAMC} & \multicolumn{1}{c}{KM-OMD} & \multicolumn{1}{c}{LDM} \\
\noalign{\smallskip} \hline \noalign{\smallskip}
\emph{promoters} & {\myfont 0.684$\pm$0.100 } & N/A & {\myfont 0.638$\pm$0.121}$\circ$ & {\myfont 0.701$\pm$0.085} & {\myfontb 0.715$\pm$0.074}$\bullet$ \\
\emph{planning-relax} & {\myfontb 0.708$\pm$0.035 } & N/A & {\myfont 0.706$\pm$0.034} & {\myfont 0.683$\pm$0.031}$\circ$ & {\myfont 0.707$\pm$0.034} \\
\emph{colic} & {\myfont 0.822$\pm$0.033 } & N/A & {\myfont 0.623$\pm$0.037}$\circ$ & {\myfont 0.825$\pm$0.024} & {\myfontb 0.841$\pm$0.018}$\bullet$ \\
\emph{parkinsons} & {\myfontb 0.929$\pm$0.029 } & N/A & {\myfont 0.852$\pm$0.036}$\circ$ & {\myfont 0.906$\pm$0.033}$\circ$ & {\myfont 0.927$\pm$0.029} \\
\emph{colic.ORIG} & {\myfont 0.638$\pm$0.043 } & N/A & {\myfont 0.623$\pm$0.027} & {\myfont 0.621$\pm$0.039} & {\myfontb 0.641$\pm$0.044} \\
\emph{sonar} & {\myfont 0.842$\pm$0.034 } & N/A & {\myfont 0.753$\pm$0.052}$\circ$ & {\myfont 0.821$\pm$0.051}$\circ$ & {\myfontb 0.846$\pm$0.032} \\
\emph{vote} & {\myfont 0.946$\pm$0.016 } & N/A & {\myfont 0.913$\pm$0.019}$\circ$ & {\myfont 0.930$\pm$0.029}$\circ$ & {\myfontb 0.968$\pm$0.013}$\bullet$ \\
\emph{house} & {\myfont 0.953$\pm$0.020 } & N/A & {\myfont 0.561$\pm$0.139}$\circ$ & {\myfont 0.938$\pm$0.022}$\circ$ & {\myfontb 0.964$\pm$0.013}$\bullet$ \\
\emph{heart} & {\myfont 0.808$\pm$0.025 } & N/A & {\myfont 0.540$\pm$0.043}$\circ$ & {\myfont 0.805$\pm$0.048} & {\myfontb 0.822$\pm$0.029}$\bullet$ \\
\emph{breast-cancer} & {\myfont 0.729$\pm$0.030 } & N/A & {\myfont 0.706$\pm$0.027}$\circ$ & {\myfont 0.691$\pm$0.024}$\circ$ & {\myfontb 0.753$\pm$0.027}$\bullet$ \\
\emph{haberman} & {\myfont 0.727$\pm$0.024 } & N/A & {\myfontb 0.742$\pm$0.021}$\bullet$ & {\myfont 0.676$\pm$0.042}$\circ$ & {\myfont 0.731$\pm$0.027} \\
\emph{vehicle} & {\myfont 0.992$\pm$0.007 } & N/A & {\myfont 0.924$\pm$0.025}$\circ$ & {\myfont 0.988$\pm$0.008}$\circ$ & {\myfontb 0.993$\pm$0.006} \\
\emph{clean1} & {\myfont 0.890$\pm$0.020 } & N/A & {\myfont 0.561$\pm$0.025}$\circ$ & {\myfont 0.772$\pm$0.043}$\circ$ & {\myfontb 0.891$\pm$0.024} \\
\emph{wdbc} & {\myfont 0.951$\pm$0.011 } & N/A & {\myfont 0.740$\pm$0.042}$\circ$ & {\myfont 0.941$\pm$0.040} & {\myfontb 0.961$\pm$0.010}$\bullet$ \\
\emph{isolet} & {\myfont 0.998$\pm$0.002 } & N/A & {\myfont 0.994$\pm$0.004}$\circ$ & {\myfont 0.995$\pm$0.003}$\circ$ & {\myfontb 0.998$\pm$0.002} \\
\emph{credit-a} & {\myfont 0.858$\pm$0.014 } & N/A & {\myfont 0.542$\pm$0.032}$\circ$ & {\myfont 0.845$\pm$0.029}$\circ$ & {\myfontb 0.861$\pm$0.013} \\
\emph{austra} & {\myfont 0.853$\pm$0.013 } & N/A & {\myfont 0.560$\pm$0.018}$\circ$ & {\myfont 0.854$\pm$0.017} & {\myfontb 0.857$\pm$0.014}$\bullet$ \\
\emph{australian} & {\myfont 0.815$\pm$0.014 } & N/A & {\myfont 0.554$\pm$0.015}$\circ$ & {\myfontb 0.860$\pm$0.014}$\bullet$ & {\myfont 0.854$\pm$0.016}$\bullet$ \\
\emph{fourclass} & {\myfontb 0.998$\pm$0.003 } & N/A & {\myfont 0.791$\pm$0.014}$\circ$ & {\myfont 0.838$\pm$0.014}$\circ$ & {\myfont 0.998$\pm$0.003} \\
\emph{german} & {\myfont 0.731$\pm$0.019 } & N/A & {\myfont 0.697$\pm$0.017}$\circ$ & {\myfont 0.742$\pm$0.017}$\bullet$ & {\myfontb 0.743$\pm$0.016}$\bullet$ \\
\noalign{\smallskip} \hline \noalign{\smallskip}
Ave. accuracy & \multicolumn{1}{c}{ {\myfont 0.844}} & \multicolumn{1}{c}{N/A} & \multicolumn{1}{c}{ {\myfont 0.701}} & \multicolumn{1}{c}{ {\myfont 0.822}} & \multicolumn{1}{c}{{\myfont 0.854}} \\
\noalign{\smallskip} \hline \noalign{\smallskip}
LDM: w/t/l & \multicolumn{1}{c}{ {\myfont 10/10/0}} & \multicolumn{1}{c}{N/A} & \multicolumn{1}{c}{ {\myfont 18/1/1}} & \multicolumn{1}{c}{ {\myfont 15/5/0}} \\
\noalign{\smallskip} \hline \hline
\end{tabular}}
\label{table 3}
\end{table*}

\begin{table*}[!t]
\scriptsize
\caption{Accuracy (mean$\pm$std.) comparison on large scale data sets. Linear kernels are used. The best accuracy on each data set is bolded. $\bullet$/$\circ$ indicates the performance is significantly better/worse than SVM (paired $t$-tests at 95\% significance level). The win/tie/loss counts are summarized in the last row. KM-OMD and MDO did not return results on some data sets in 48 hours.}
\centering
\vskip 0.1in
\scalebox{1.15}[1.15]{
\begin{tabular}{ l | l l l l l }
\hline \hline \noalign{\smallskip}
Dataset & \multicolumn{1}{c}{SVM} & \multicolumn{1}{c}{MDO} & \multicolumn{1}{c}{MAMC} & \multicolumn{1}{c}{KM-OMD} & \multicolumn{1}{c}{LDM} \\
\noalign{\smallskip} \hline \noalign{\smallskip}
\emph{farm-ads} & {\myfont 0.880$\pm$0.007} & {\myfont 0.880$\pm$0.007} & {\myfont 0.759$\pm$0.038}$\circ$ & \multicolumn{1}{c}{N/A} & {\myfontb 0.890$\pm$0.008}$\bullet$ \\
\emph{news20} & {\myfont 0.954$\pm$0.002} & {\myfont 0.948$\pm$0.002}$\circ$ & {\myfont 0.772$\pm$0.017}$\circ$ & \multicolumn{1}{c}{N/A} & {\myfontb 0.960$\pm$0.001}$\bullet$ \\
\emph{adult-a} & {\myfont 0.845$\pm$0.002} & {\myfont 0.788$\pm$0.053}$\circ$ & {\myfont 0.759$\pm$0.002}$\circ$ & \multicolumn{1}{c}{N/A} & {\myfontb 0.846$\pm$0.003}$\bullet$ \\
\emph{w8a} & {\myfont 0.983$\pm$0.001} & {\myfontb 0.985$\pm$0.001}$\bullet$ & {\myfont 0.971$\pm$0.001}$\circ$ & \multicolumn{1}{c}{N/A} & {\myfont 0.983$\pm$0.001} \\
\emph{cod-rna} & {\myfontb 0.899$\pm$0.001} & {\myfont 0.774$\pm$0.203} & {\myfont 0.667$\pm$0.001}$\circ$ & \multicolumn{1}{c}{N/A} & {\myfont 0.899$\pm$0.001} \\
\emph{real-sim} & {\myfont 0.961$\pm$0.001} & {\myfont 0.955$\pm$0.002}$\circ$ & {\myfont 0.744$\pm$0.004}$\circ$ & \multicolumn{1}{c}{N/A} & {\myfontb 0.971$\pm$0.001}$\bullet$ \\
\emph{ijcnn1} & {\myfont 0.921$\pm$0.003} & {\myfontb 0.921$\pm$0.002} & {\myfont 0.904$\pm$0.001}$\circ$ & \multicolumn{1}{c}{N/A} & {\myfont 0.921$\pm$0.002} \\
\emph{skin} & {\myfont 0.934$\pm$0.001} & {\myfont 0.929$\pm$0.003}$\circ$ & {\myfont 0.792$\pm$0.000}$\circ$ & \multicolumn{1}{c}{N/A} & {\myfontb 0.934$\pm$0.001} \\
\emph{covtype} & {\myfont 0.762$\pm$0.001} & {\myfont 0.760$\pm$0.003}$\circ$ & {\myfont 0.628$\pm$0.002}$\circ$ & \multicolumn{1}{c}{N/A} & {\myfontb 0.763$\pm$0.001} \\
\emph{rcv1} & {\myfont 0.969$\pm$0.000} & {\myfont 0.959$\pm$0.000}$\circ$ & {\myfont 0.913$\pm$0.000}$\circ$ & \multicolumn{1}{c}{N/A} & {\myfontb 0.977$\pm$0.000}$\bullet$ \\
\emph{url} & {\myfontb 0.993$\pm$0.006} & {\myfont 0.993$\pm$0.006} & {\myfont 0.670$\pm$0.000}$\circ$ & \multicolumn{1}{c}{N/A} & {\myfont 0.993$\pm$0.006} \\
\emph{kdd2010} & {\myfont 0.852$\pm$0.001} & \multicolumn{1}{c}{N/A} & {\myfont 0.853$\pm$0.000}$\bullet$ & \multicolumn{1}{c}{N/A} & {\myfontb 0.881$\pm$0.001}$\bullet$ \\
\noalign{\smallskip} \hline \noalign{\smallskip}
Ave. accuracy & \multicolumn{1}{c}{ {\myfont 0.913}} & \multicolumn{1}{c}{ {\myfont 0.899}} & \multicolumn{1}{c}{ {\myfont 0.786}} & \multicolumn{1}{c}{N/A} & \multicolumn{1}{c}{ {\myfont 0.919}} \\
\noalign{\smallskip} \hline \noalign{\smallskip}
LDM: w/t/l & \multicolumn{1}{c}{{\myfont 6/6/0}} & \multicolumn{1}{c}{{\myfont 7/3/1}} & \multicolumn{1}{c}{{\myfont 12/0/0}} & \multicolumn{1}{c}{N/A} \\
\noalign{\smallskip} \hline \hline
\end{tabular}}
\label{table 4}
\end{table*}

LDMs are compared with standard SVMs which ignore the margin distribution, and three state-of-the-art methods, that is, Margin Distribution Optimization (MDO) \cite{Garg2003Margin}, Maximal Average Margin for Classifiers (MAMC) \cite{Pelckmans2007A} and Kernel Method for the direct Optimization of the Margin Distribution (KM-OMD) \cite{Aiolli2008A}. For SVM, KM-OMD and LDM, the regularization parameter $C$ is selected by 5-fold cross validation from $[10, 50, 100]$. For MDO, the parameters are set as the recommended parameters in \cite{Garg2003Margin}. For LDM, the regularization parameters $\lambda_1$, $\lambda_2$ are selected by 5-fold cross validation from the set of $[2^{-8}, \ldots, 2^{-2}]$, the parameters $\eta_t$ and $t_0$ are set with the same setup in \cite{Xu2010Towards}, and $T$ is fixed to $5$. The width of the RBF kernel for SVM, MAMC, KM-OMD and LDM are selected by 5-fold cross validation from the set of $[2^{-2}\delta, \ldots, 2^2\delta]$, where $\delta$ is the average distance between instances. All selections are performed on training sets.

\begin{figure*}[!t]
\begin{center}
\centerline{\includegraphics[width=\columnwidth]{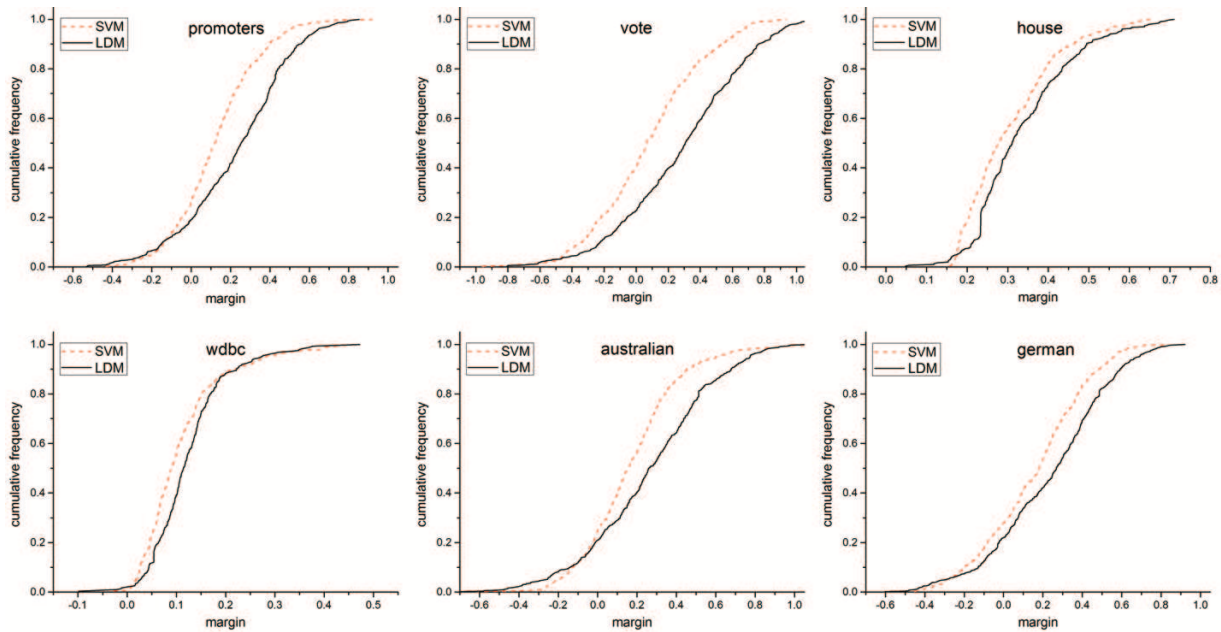}}
\caption{Cumulative frequency ($y$-axis) with respect to margin ($x$-axis) of SVM and LDM on some representative regular scale data sets. The more right the curve, the larger the accumulated margin.}
\label{fig:Margin Distributions}
\end{center}
\vskip -0.2in
\end{figure*}

\subsection{Results on Regular Scale Data Sets} \label{sec:Resultsb}

Tables \ref{table 2} and \ref{table 3} summarize the results on twenty regular scale data sets. As can be seen, the overall performance of LDM is superior or highly competitive to SVM and other compared methods. Specifically, for linear kernel, LDM performs significantly better than SVM, MDO, MAMC, KM-OMD on 12, 9, 17 and 10 over 20 data sets, respectively, and achieves the best accuracy on 13 data sets; for RBF kernel, LDM performs significantly better than SVM, MAMC, KM-OMD on 10, 18 and 15 over 20 data sets, respectively, and achieves the best accuracy on 15 data sets. MDO is not compared since it is specified for the linear kernel. In addition, as can be seen, in comparing with standard SVM which does not consider margin distribution, the win/tie/loss counts show that LDM is always better or comparable, never worse than SVM.

\subsection{Results on Large Scale Data Sets} \label{sec:Resultsl}

Table \ref{table 4} summarizes the results on twelve large scale data sets. KM-OMD did not return results on all data sets and MDO did not return results on KDD2010 in 48 hours due to the high computational cost. As can be seen, the overall performance of LDM is superior or highly competitive to SVM and other compared methods. Specifically, LDM performs significantly better than SVM, MDO, MAMC on 6, 7 and 12 over 12 data sets, respectively, and achieves the best accuracy on 8 data sets. In addition, the win/tie/loss counts show that LDM is always better or comparable, never worse than SVM.

\subsection{Margin Distributions} \label{sec:IMD}

Figure \ref{fig:Margin Distributions} plots the cumulative margin distribution of SVM and LDM on some representative regular scale data sets. The curves for other data sets are similar. The point where a curve and the $x$-axis crosses is the corresponding minimum margin. As can be seen, LDM usually has a little bit smaller minimum margin than SVM, whereas the LDM curve generally lies on the right side, showing that the margin distribution of LDM is generally better than that of SVM. In other words, for most examples, LDM generally produce a larger margin than SVM.

\begin{figure*}[!t]
\begin{center}
\centerline{\includegraphics[width=\columnwidth]{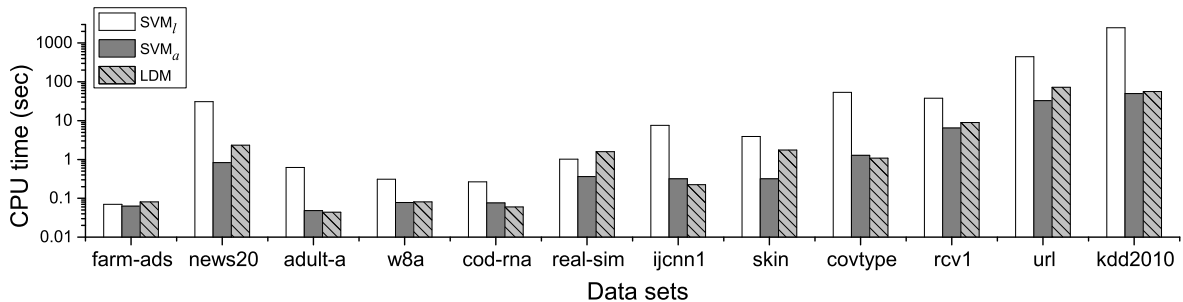}}
\caption{CPU time on the large scale data sets.}
\label{fig: Running Time}
\end{center}
\vskip -0.2in
\end{figure*}

\subsection{Time Cost} \label{sec:tc}

We compare the time cost of LDM and SVM on the twelve large scale data sets. All the experiments are performed with MATLAB 2012b on a machine with 8\(\times\)2.60 GHz CPUs and 16GB main memory. The average CPU time (in seconds) on each data set is shown in Figure \ref{fig: Running Time}. We denote SVM implemented by the LIBLINEAR \cite{Fan2008LIBLINEAR} package as SVM$_l$ and SVM implemented by ASGD\footnote{http://leon.bottou.org/projects/sgd} as SVM$_a$, respectively. It can be seen that, both SVM$_a$ and LDM are much faster than SVM$_l$, owing to the use of ASGD. LDM is just slightly slower than SVM$_a$ on three data sets (news20, real-sim and skin) but highly competitive with SVM$_a$ on the other nine data sets. Note that both SVM$_l$ and SVM$_a$ are very fast implementations of SVMs; this shows that LDM is also computationally efficient.

\subsection{Parameter Influence} \label{sec:pi}
LDM has three regularization parameters, i.e., $\lambda_1$, $\lambda_2$ and $C$. In previous empirical studies, they are set according to cross validation. Figure \ref{fig: pi} further studies the influence of them on some representative regular scale data sets by fixing other parameters. Specifically, Figure \ref{fig: pi}(a) shows the influence of $\lambda_1$ on the accuracy by varying it from $2^{-8}$ to $2^{-2}$ while fixing $\lambda_2$ and $C$ as the value suggested by the cross validation described in Section \ref{sec:experimental setup}. Figure \ref{fig: pi}(b) and Figure \ref{fig: pi}(c) are obtained in the same way. It can be seen that, the performance of LDM is not very sensitive to the setting of the parameters, making LDM even more attractive in practice.

\begin{figure*}[!t]
\begin{center}
\centerline{\includegraphics[width=\columnwidth]{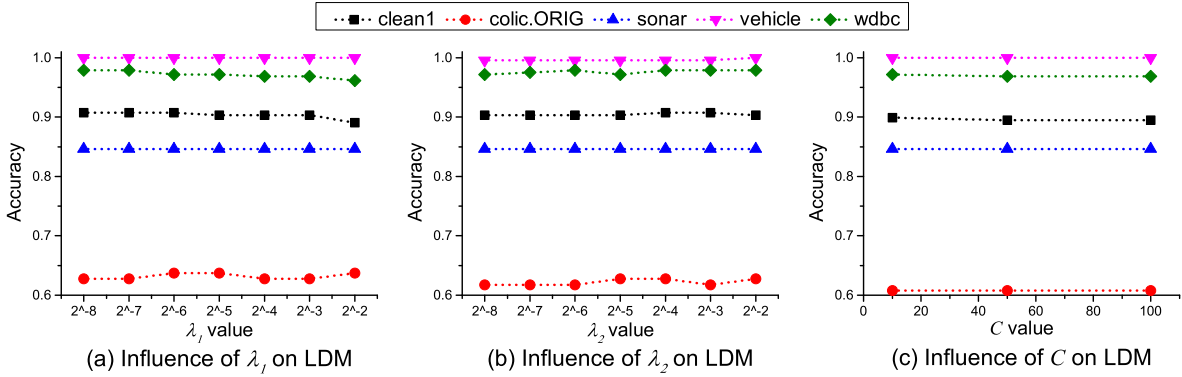}}
\caption{Parameter influence on some representative regular scale data sets.}
\label{fig: pi}
\end{center}
\vskip -0.2in
\end{figure*}

\section{Related Work} \label{sec:Related Work}

There are a few studies considered margin distribution in SVM-like algorithms \cite{Garg2003Margin,Pelckmans2007A,Aiolli2008A}. Garg et al. \cite{Garg2003Margin} proposed the Margin Distribution Optimization (MDO) algorithm which minimizes the sum of the cost of each instance, where the cost is a function which assigns larger values to instances with smaller margins. MDO can be viewed as a method of optimizing weighted margin combination, where the weights are related to the margins. The objective function optimized by MDO, however, is non-convex, and thus, it may get stuck in local minima. In addition, MDO can only be used for linear kernel. As our experiments in Section 4 disclosed, the performance of MDO is inferior to LDM.

Pelckmans et al. \cite{Pelckmans2007A} proposed the Maximal Average Margin for Classifiers (MAMC) and it can be viewed as a special case of LDM assuming that the margin variance is zero. MAMC has a closed-form solution, however, it will degenerate to a trivial solution when the classes are not with equal sizes. Our experiments in Section 4 showed that LDM is clearly superior to MAMC.

Aiolli et al. \cite{Aiolli2008A} proposed a Kernel Method for the direct Optimization of the Margin Distribution (KM-OMD) from a game theoretical perspective. Similar to MDO, this method also directly optimizes a weighted combination of margins over the training data, ignoring the influence of margin variances. Besides, this method considers hard-margin only, which may be another reason why it behaves worse than our method. It is noteworthy that the computational cost prohibits KM-OMD to be applied to large scale data, as shown in Table 4.

\section{Conclusions}

Support vector machines work by maximizing the minimum margin. Recent theoretical results suggested that the margin distribution, rather than a single-point margin such as the minimum margin, is more crucial to the generalization performance. In this paper, we propose the large margin distribution machine (LDM) which tries to optimize the margin distribution by maximizing the margin mean and minimizing the margin variance simultaneously. The LDM is a general learning approach which can be used in any place where SVM can be applied. Comprehensive experiments on twenty regular scale data sets and twelve large scale data sets validate the superiority of LDM to SVMs and many state-of-the-art methods. In the future it will be interesting to generalize the idea of LDM to regression and other learning settings.

\bibliographystyle{plain}
\bibliography{Large_Margin_Distribution_Machine}

\end{document}